\documentclass[10pt,twocolumn,letterpaper]{article}

\usepackage[pagenumbers]{cvpr} 

\usepackage[accsupp]{axessibility} 

\definecolor{cvprblue}{rgb}{0.21,0.49,0.74}
\usepackage[pagebackref,breaklinks,colorlinks,allcolors=cvprblue]{hyperref}

\usepackage{hyperref}
\usepackage{url}
\usepackage{xcolor}
\usepackage{booktabs}
\usepackage{pifont}
\usepackage{makecell}
\usepackage{placeins}
\usepackage{graphicx}
\usepackage{booktabs}
\usepackage{tcolorbox}
\usepackage{stfloats}

\usepackage{colortbl}
\usepackage{xcolor}
\usepackage{makecell}

\usepackage{amsthm}  
\usepackage{multirow} 
\usepackage{wrapfig}

\usepackage{float} 

\usepackage{url}

\newtheorem{lemma}{Lemma}

\newcommand{\bfsection}[1]{\vspace*{0.1cm}\noindent\textbf{#1.}}

\newcommand{\citeneed}[1][]{\textbf{\textcolor{red}{[CITE]}}}

\definecolor{cvprblue}{rgb}{0.21,0.49,0.74}
\usepackage[pagebackref,breaklinks,colorlinks,allcolors=cvprblue]{hyperref}

\def\paperID{7426} 
\def\confName{CVPR}
\def\confYear{2026}

\title{NexusFlow: Unifying Disparate Tasks under Partial Supervision \\via Invertible Flow Networks}

\author{
\textbf{Fangzhou Lin}$^{1,2,3}$\quad
\textbf{Yuping Wang}$^{4}$\quad
\textbf{Yuliang Guo}$^{5}$\footnotemark[2]\quad
\textbf{Zixun Huang}$^{5}$\quad
\textbf{Xinyu Huang}$^{5}$\quad\\[2pt]
\textbf{Haichong Zhang}$^{2}$\quad
\textbf{Kazunori Yamada}$^{3}$\quad
\textbf{Zhengzhong Tu}$^{2}$\footnotemark[1]\quad
\textbf{Liu Ren}$^{5}$\quad
\textbf{Ziming Zhang}$^{1}$\footnotemark[1]\\[4pt]
\footnotesize
$^{1}$ Worcester Polytechnic Institute\;
$^{2}$Texas A\&M University \;
$^{3}$Tohoku University \\
\footnotesize
$^{4}$University of Michigan  \;
\footnotesize
$^{5}$Bosch Research North America \& Bosch Center for AI \;\\
\footnotesize
\footnotesize
\texttt{https://github.com/ark1234/NexusFlow}
}

\makeatletter
\let\@oldmaketitle\@maketitle
\renewcommand{\@maketitle}{\@oldmaketitle
\vspace{-15pt}
\centering\includegraphics[trim=0mm 0mm 0mm 0mm, clip, width=1.0\linewidth]{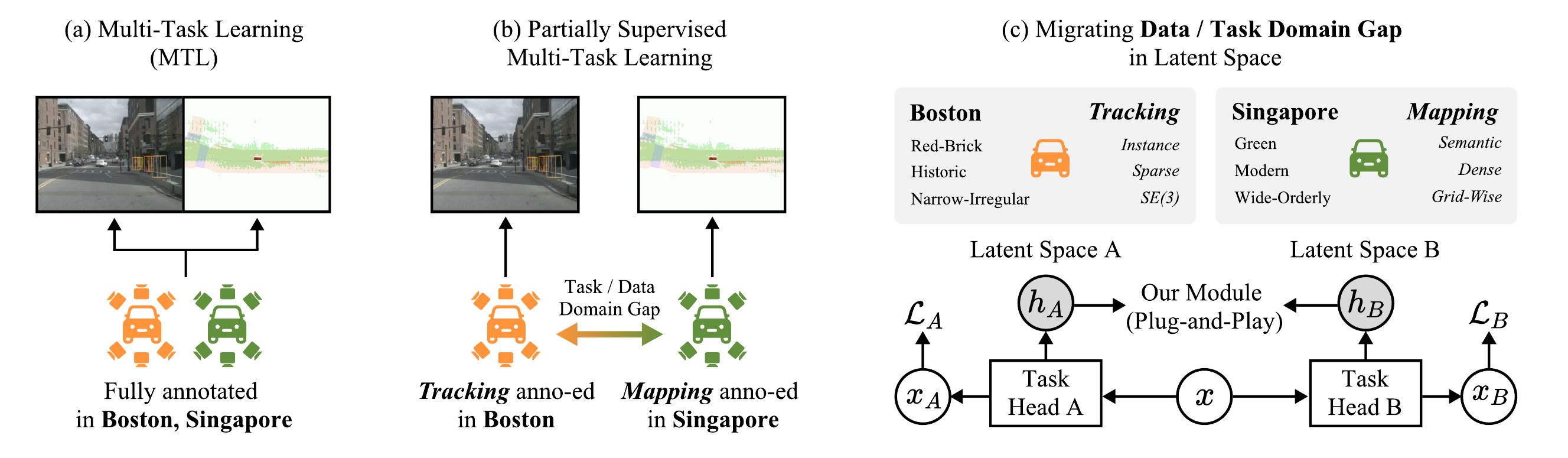}
\captionof{figure}{
\textbf{Problem Setup \& Motivation.} 
We illustrate the setup of \textit{Partially Supervised Multi-Task Learning} using autonomous driving as a representative example.  
(a) In the ideal case, all training data are fully annotated for all tasks (e.g., tracking and mapping), and mixed training achieves the upper bound of Multi-Task Learning (MTL) performance. 
(b) In practice, however, collaborative communities often provide large amounts of valuable but single-task-oriented datasets. These datasets usually differ in both task labels and domains (e.g., geographic or scene-level domain gaps). 
Naively mixing them leads to incomplete supervision, causing degraded performance and poor cross-domain generalization compared with fully supervised MTL. 
(c) Our goal is to bridge both task and data domain gaps using a simple yet effective migration strategy in the latent space.
}
\vspace{20pt}
\label{fig:teaser}}

\begin{document}
\def\customfootnotetext#1#2{{%
  \let\thefootnote\relax
  \footnotetext[#1]{#2}}}

\maketitle

\customfootnotetext{2}
{\textsuperscript{$\dagger$}Project lead: Yuliang Guo}
\customfootnotetext{1}{\textsuperscript{*}Corresponding Author: Zhengzhong Tu and Ziming Zhang}

\begin{abstract}

Partially Supervised Multi-Task Learning (PS-MTL) aims to leverage knowledge across tasks when annotations are incomplete. Existing approaches, however, have largely focused on the simpler setting of homogeneous, dense prediction tasks, leaving the more realistic challenge of learning from structurally diverse tasks unexplored. To this end, we introduce \textbf{NexusFlow}, a novel, lightweight, and plug-and-play framework effective in \textit{both} settings. NexusFlow introduces a set of surrogate networks with invertible coupling layers to align the latent feature distributions of tasks, creating a unified representation that enables effective knowledge transfer. The coupling layers are bijective, preserving information while mapping features into a shared canonical space. This invertibility avoids representational collapse and enables alignment across structurally different tasks without reducing expressive capacity.
We first evaluate NexusFlow on the core challenge of domain-partitioned autonomous driving, where dense map reconstruction and sparse multi-object tracking are supervised in different geographic regions, creating both structural disparity and a strong domain gap. NexusFlow sets a new state-of-the-art result on nuScenes, outperforming strong partially supervised baselines. To demonstrate generality, we further test NexusFlow on NYUv2 using three homogeneous dense prediction tasks, segmentation, depth, and surface normals, as a representative N-task PS-MTL scenario. NexusFlow yields consistent gains across all tasks, confirming its broad applicability. 

\end{abstract}    
\section{Introduction}
\label{sec:intro}

Learning multiple tasks simultaneously via Multi-Task Learning (MTL) is a powerful paradigm for improving model efficiency and generalization while avoiding redundant training compared with single task~\cite{xu2018pad,zhang2019pattern,vandenhende2020mti,xu2022fpcc,lin2023hyperbolic,yue2024understanding,li2024region,wu2026consid}. However, its real-world applicability is often limited by the prohibitive cost of acquiring exhaustive annotations for every task, particularly in vision-heavy domains. In practice, datasets frequently lack labels for some tasks, or contain incomplete and unreliable annotations, motivating the crucial research direction of Partially Supervised Multi-Task Learning (PS-MTL)~\cite{li2022learning}. 

Significant progress has been made when tasks are homogeneous dense predictions, thanks to their natural interdependence~\cite{zamir2018taskonomy,zamir2020robust}. This has enabled joint training in both partially and fully supervised settings for tasks such as semantic segmentation, depth estimation, and surface normals estimation~\cite{eigen2014depth,he2017mask,chen2018gradnorm,poggi2020uncertainty,zhang2020uc}. Yet, when tasks are structurally disparate (e.g., one requiring dense pixel-wise labels and another producing sparse instance-level outputs), the challenge becomes far greater and has received very limited attention. 

The challenge becomes even more pronounced under realistic supervision patterns. Existing PS-MTL methods typically simulate missing labels by randomly masking task annotations~\cite{li2024region}, resulting in a controlled but overly simplified setup~\cite{ye2024diffusionmtl}. In practice, however, annotations for different tasks are often collected in disjoint domains (e.g., distinct geographic regions or capture environments). This creates a strong decoupling between \textit{task type} and \textit{data domain}, introducing substantial domain shift on top of structural disparity. Such domain-partitioned, heterogeneous-task conditions are common in real-world multi-task systems but remain largely unexplored. To the best of our knowledge, no prior work systematically addresses PS-MTL in this challenging and practically important regime.

\begin{figure}
    \centering
    \includegraphics[width=1.0\linewidth]{Figure/nexusflow_pipe.pdf}
    \vspace{-10pt}
    \caption{
    \textbf{NexusFlow Pipeline.} 
    A simple yet general framework that scales to the $N$-task partially supervised MTL setting. 
    Given a data batch where only \textit{Task 1} has annotations, we extract latent features $h_i$ from all $N$ task heads. These features are then encoded into a shared representation space $\{z_i\}$ via an invertible network, where the model minimizes the distance between each $z_i$ and their mean to promote cross-task consistency.
    }
    \label{fig:pipeline}
    \vspace{-16pt}
\end{figure}

Motivated by the challenge of learning heterogeneous tasks under domain-partitioned partial supervision, we introduce \textbf{NexusFlow}, a simple yet effective mechanism that aligns latent feature distributions during standard backpropagation. Instead of modifying task heads or requiring architectural coupling, NexusFlow provides a lightweight, plug-and-play module that operates purely in feature space (Figure~\ref{fig:teaser}), making it broadly applicable to diverse multi-task systems.

Inspired by flow-based invertible models~\cite{Dinh2014NICE,DinhSB17:realnvp,ardizzone2018analyzing,behrmann2019invertible,kobyzev2020normalizing}, NexusFlow inserts, for each task, a dimensionality reduction module followed by invertible affine coupling layers (Figure~\ref{fig:pipeline}). These coupling layers are bijective: they preserve all task-relevant information while enabling flexible transformations into a shared canonical space. This invertibility is essential; unlike conventional CNN-based alignment modules, the coupling layers prevent representational collapse beyond the reduced MLP embedding and allow structurally different tasks (e.g., sparse instance sets vs.\ dense geometric fields) to be aligned without sacrificing expressive capacity. By explicitly minimizing distributional discrepancies within this canonical space, NexusFlow encourages the model to discover cross-task feature representations that generalize even when each task is supervised in disjoint domains. This leads to a unified latent space that is more robust to both structural disparity and domain shift.

We validate NexusFlow in two complementary settings:  
(1) a structurally disparate, domain-partitioned two-task perception system on the nuScenes dataset~\cite{caesar2020nuscenes}, involving multi-object tracking and map reconstruction under PS-MTL; and  
(2) a homogeneous three-task PS-MTL setup on NYUv2~\cite{silberman2012indoor}, covering dense prediction tasks of semantic segmentation, depth estimation, and surface normal estimation. Across both scenarios, NexusFlow not only delivers substantial gains over strong baselines but also surpasses baselines integrated with prior state-of-the-art PS-MTL modules developed for homogeneous tasks with random label masking, demonstrating strong generalization and broad applicability.

    



Our key contributions are as follows:

\begin{itemize}[nosep, leftmargin=*]

\item We are, to the best of our knowledge, the \textit{first} to systematically formulate and tackle Partially Supervised Multi-Task Learning (PS-MTL) for \textit{structurally disparate} tasks under \textit{domain-partitioned} supervision, an important yet largely unexplored problem for modern perception systems.

\item We introduce \textbf{NexusFlow}, a lightweight, plug-and-play framework that addresses this challenge using \textbf{independent invertible coupling layers} to align task-specific feature distributions and enable effective cross-task knowledge transfer.

\item We provide both theoretical and empirical analysis showing that our invertible transformations enable reliable alignment of feature distributions across heterogeneous tasks while alleviating dimensionality collapse.

\item We conduct extensive experiments across autonomous driving and indoor dense prediction benchmarks, demonstrating the effectiveness and generality of NexusFlow.



\end{itemize}
\section{Related Work}

\bfsection{Partially Supervised Multi-Task Learning} Partially Supervised Multi-Task Learning (PS-MTL) tackles the challenge of learning multiple tasks when only a subset has labels per sample. Early approaches showed success mainly in \textbf{homogeneous dense prediction tasks} (e.g., semantic segmentation, depth estimation), using strategies like consistency regularization, pseudo-labeling, or adversarial training. For example, adversarial discriminators have been used to align distributions across partially labeled datasets~\cite{wang2022semi}, while consistency-based methods either employ model augmentation to provide pseudo-supervision~\cite{spinola2023knowledge} or regularize cross-task relationships in a shared space~\cite{li2022learning}. Pseudo-labeling has also been advanced via hierarchical task tokens for “label discovery,” enabling dense supervision for unlabeled tasks~\cite{zhang2024multi}.

Despite these advances, such techniques remain tailored to homogeneous tasks and are not directly applicable to \textbf{structurally heterogeneous settings} like object tracking versus map segmentation. Even recent unified frameworks for autonomous driving~\cite{hu2023planning, huang2023fuller,wang2025uniocc,xing2025openemma} assume full supervision, leaving unresolved the realistic case of partially annotated, structurally disparate tasks—where annotation costs are prohibitive~\cite{dulac2021challenges,lin2023infocd,vettoruzzo2024advances,zhang2025gps,dai2025language}. A central challenge is learning shared representations without negative transfer. Architectural methods explore optimal layer sharing~\cite{ruder2019latent}, while alignment-based methods project features into a common latent space. Early pairwise approaches~\cite{li2022learning} struggled with quadratic complexity, motivating more scalable solutions like JTR~\cite{nishi2024joint}, which stacks all predictions into a unified joint-task space, or StableMTL~\cite{cao2025stablemtl}, which applies a unified latent loss with efficient 1-to-N attention.

Beyond scalability, newer methods seek richer cross-task knowledge transfer. Region-aware strategies use SAM~\cite{kirillov2023segment} to detect local regions and model their features as Gaussian distributions, enabling fine-grained alignment~\cite{li2024region}. DiffusionMTL~\cite{ye2024diffusionmtl} reframes partially labeled outputs as noisy predictions and refines them through a denoising diffusion process with multi-task conditioning. These works highlight the trend towards more sophisticated, scalable PS-MTL approaches, but none address the harder setting of \textbf{structurally disparate tasks under domain-partitioned supervision}, which is the focus of our work.

\bfsection{Multi-Task Learning for Autonomous Driving Perception} Autonomous driving perception has shifted from modular pipelines (separate detection, tracking, and mapping) to integrated multi-task learning (MTL) frameworks~\cite{wang2025generative,wang2025uniocc,gao2025safecoop}, motivated by efficiency, performance, and shared representations~\cite{xu2022fpcc,luo2025v2x,xing2024autotrust,xing2025openemma}.

A major breakthrough is end-to-end models that unify perception, prediction, and planning. UniAD~\cite{hu2023planning} pioneered this direction with a Bird’s-Eye-View (BEV) representation and Transformer-based modules for tracking (TrackFormer), mapping (MapFormer), motion prediction (MotionFormer), and occupancy prediction (OccFormer), all linked by a query-based mechanism. Building on this, GenAD~\cite{zheng2024genad} models the scene generatively for joint prediction and planning, while DriveTransformer~\cite{jia2025drivetransformer} parallelizes tasks via shared attention for scalability. VAD~\cite{jiang2023vad} instead uses a fully vectorized scene representation, improving efficiency and reducing collision rates.

A key challenge is \textbf{domain generalization}, as perception models face shifts in weather, lighting, sensors, and geography. ~\cite{wang2021generalizing} highlights this as central to real-world ML. For BEV-based systems, ~\cite{wang2023towards} analyzes domain gaps in 3D detection and proposes robust depth learning. ~\cite{jiang2024bev} introduces DA-BEV for unsupervised adaptation, combining image-view and BEV features. ~\cite{chang2024unified} unifies domain generalization and adaptation with multi-view overlap depth constraints. After all, the \textbf{cost of annotation} remains prohibitive: datasets are often partially labeled. Li ~\cite{li2022learning} addresses this by mapping task pairs into a joint space to enable sharing under incomplete supervision, a critical issue for large-scale autonomous driving systems where exhaustive labels are impractical.
\section{Approach}\label{sec:approach}

\subsection{Preliminaries}

\bfsection{Notations \& Problem Definition} \label{prob_def} We formulate our problem within the framework of Partially Supervised Multi-Task Learning (PS-MTL). We consider a set of $n$ tasks, $\mathcal{T} = \{\mathcal{T}_1, \dots, \mathcal{T}_n\}$. Our goal is to train a single unified model that shares knowledge across these tasks.

Following the standard MTL paradigm, our model consists of a shared feature encoder and $n$ task-specific decoders (or "heads"). Given an input (can be either image or frames), the encoder produces a shared latent representation. Each head then processes this representation to produce a task-specific prediction. 

The core difficulty of PS-MTL originates from (1) the nature of the tasks $\mathcal{T}$ and (2) the distribution of the supervision masks $\mathbf{m}$. In this paper, our main focus is to tackle \textbf{Domain-Partitioned, Structurally Disparate PS-MTL (Core Challenge)}.
This setting represents the central focus of our work and introduces significant real-world challenges absent in previous methods.
\begin{itemize}[leftmargin=*,itemsep=0pt,topsep=0pt,parsep=0pt]
    \item \textbf{Structurally Disparate Tasks:} The tasks are fundamentally different in their output structure. This disparity is most evident in their output spaces:
    \begin{itemize}[leftmargin=*,itemsep=0pt]
        \item Map Reconstruction ($\mathcal{T}_{map}$): a dense, grid-based representation, where each spatial location is assigned a semantic class (e.g., lane, divider, drivable area).
        \item Multi-Object Tracking($\mathcal{T}_{track}$): a sparse, instance-level set, where each object is represented by a bounding box and an identity, with the number of objects varying across samples.
    \end{itemize}
    \item \textbf{Homogeneous Tasks:} All tasks are of a similar nature, specifically dense predictions. For example, $K=3$ with $\mathcal{T}_1$ being semantic segmentation, $\mathcal{T}_2$ being depth estimation  and $\mathcal{T}_3$ being surface normals estimation. All tasks produce a dense output. The PS-MTL under this setup was mostly explored in the earlier methods.
    \item \textbf{Domain-Partitioned Supervision:} This is a far more challenging scenario than random masking. Supervision is assigned according to geographically distinct subsets of data. Specifically, in our experiments on the nuScenes dataset~\cite{caesar2020nuscenes}, one task is annotated \textit{only} in Boston scenes, while the other is annotated \textit{only} in Singapore scenes. This means for any given sample, the mask is \textit{strict}, and is perfectly correlated with the input's geographic domain. This creates a substantial \textit{domain gap} between supervision sources, which, when combined with structural disparity, dramatically increases the difficulty of knowledge transfer. This setup also applies to indoor domains with domain gaps between regionally distinct subsets. 
\end{itemize}

\subsection{NexusFlow}
\label{sec:nexusflow_general}
Given a partially-supervised multi-task learning (PS-MTL) setting composed of 
$n$ tasks 
$\{\mathcal{T}_1, \mathcal{T}_2, \dots, \mathcal{T}_n\}$,
our goal is to improve cross-task consistency by aligning latent
representations while keeping the baseline architecture unchanged.
Figure~\ref{fig:pipeline} provides an overview of the unified design.

Let the baseline model contain a shared encoder (or any shared representation
module), followed by $n$ task-specific branches.
We denote the intermediate feature for task $t_i$ as $f^{(i)}$.
NexusFlow augments this baseline with lightweight plug-and-play surrogate
modules that convert these intermediate features into a common latent space for
distribution alignment.

\paragraph{Surrogate Modules.}
For each task $t_i$, NexusFlow adds a surrogate module  
$S_{\mathrm{surro}_i}(\cdot)$ composed of two components:

\textbf{(1) Feature Aggregator.}
Each task feature $h_{i}$ is compressed into a fixed-dimensional embedding: $h'_{i} = g_i\!\left(h_{i}\right),$ where $g_i(\cdot)$ is a lightweight feature aggregator  
(e.g., MLP, deformable attention).  
Importantly, the baseline forward path remains entirely unaffected.

\textbf{(2) Invertible Transformation.}
Each embedding is transformed by an affine coupling layer
$c_i(\cdot)$ from normalizing flows~\cite{Dinh2014NICE}: $z_{i} = c_i\!\left(h'_{i}\right)$. We use the standard affine coupling transform introduced in RealNVP~\cite{DinhSB17:realnvp}. Given an embedding $h'_{i}$ split into $ (h_{i}^{'1}, h_{i}^{'2})$, the layer performs:

\begin{align}
c(h_{i}) &= 
\bigl(\mathbf\; h_{i}^{'2} \odot \exp\bigl(s(h_{i}^{'1})\bigr) + t(h_{i}^{'1})\bigr),
\end{align}
where $s(\cdot)$ and $t(\cdot)$ are small MLPs. Due to the design of the coupling transform, its inverse admits a closed-form expression that directly reuses the forward $s(\cdot)$ or $t(\cdot)$ functions, without requiring inversion of these MLPs~\cite{DinhSB17:realnvp}.  
This bijective transformation maps features into a canonical latent space while 
preserving information content.

\paragraph{Distribution Alignment Objective.}
Let $\mathcal{L}_{t_1}, \dots, \mathcal{L}_{t_n}$ denote the original task
losses of the baseline, which remain unchanged.

NexusFlow introduces an auxiliary alignment loss:
\begin{equation}
\mathcal{L}_{\mathrm{align}}
= \mathrm{Align}\bigl(z^{1}, z^{2}, \dots, z^{n}\bigr).
\end{equation}

We consider two general formulations:

\textbf{(A) Pairwise Matching:}
\begin{equation}
\mathcal{L}_{\mathrm{align(pair)}} 
= \sum_{i<j} \| z^{i} - z^{j} \|_2^2.
\end{equation}

\textbf{(B) Center-Based Matching:}
First compute the latent center: $\bar{z} = \frac{1}{n} \sum_{i=1}^{n} z^{i},$

then align each latent to it:
\begin{equation}
\mathcal{L}_{\mathrm{align(center)}}
= \sum_{i=1}^{n} \| z^{i} - \bar{z} \|_2^2.
\end{equation}


The pairwise formulation requires $O(n^2)$ pairwise comparisons, whereas the center-based formulation uses only $O(n)$ terms, we adopt the center-based objective for its substantially higher efficiency and simpler gradient structure. For completeness, we also compare both variants in Table~\ref{tab:nyuv2full}.

\paragraph{Total Training Objective.}
We weigh $\mathcal{L}_{\mathrm{align}}$ by $\lambda$ and add it to the task losses to form the total loss:
\begin{equation}
\mathcal{L}_{\mathrm{all}}
=
\sum_{i=1}^{n} \mathcal{L}_{t_i}
\;+\;
\lambda\, \mathcal{L}_{\mathrm{align}}.
\end{equation}

NexusFlow can be applied either to {one-phase joint training}, where the
alignment loss is active throughout, or to a {two-phase fine-tuning}
strategy. As shown in Table~\ref{tab:sota-combined}, both yield consistent improvements,
demonstrating the stability and plug-and-play nature of the approach.

\begin{figure*}[t]
\centering
\includegraphics[width=1.0\linewidth]
{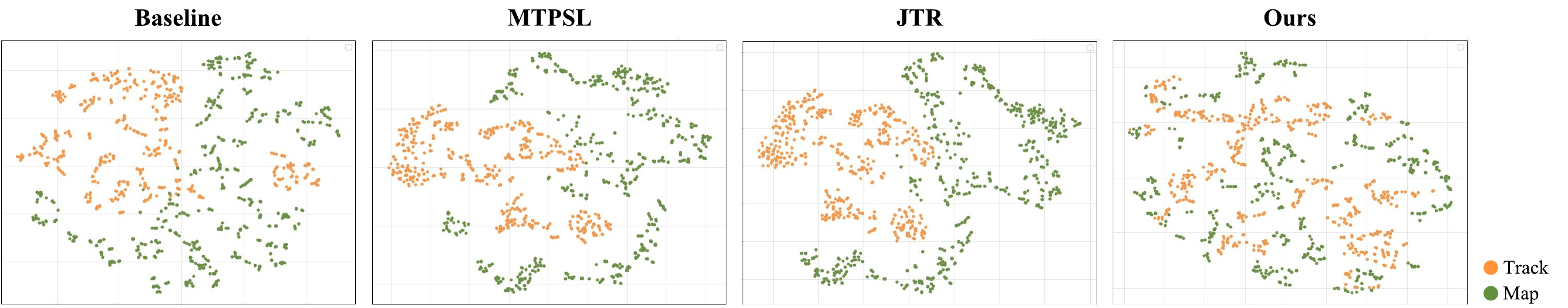}
\caption{\small{From left to right: t-SNE visualizations from the coupling layers of Baseline, MTPSL, JTR, and NexusFlow(Ours).}}

\label{fig:tsne_analysis}

\end{figure*}

\begin{figure*}[t]
\centering
\includegraphics[width=1.0\linewidth]
{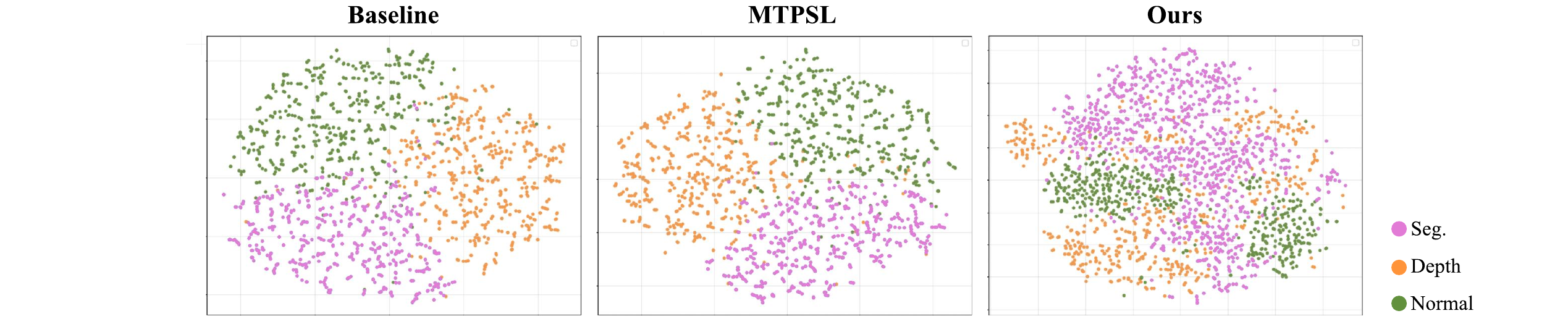}
\caption{{From left to right: t-SNE visualizations from the coupling layers of Baseline, MTPSL, and NexusFlow (Ours).}}

\label{fig:tsne_analysis_2}
\vspace{-15pt}
\end{figure*}

\subsection{Theoretical Analysis}
\label{ssec:theory}

We present a formal analysis of our alignment mechanism.  
Its effectiveness relies critically on the invertibility of the coupling layers
$c(\cdot)$, which ensures a one-to-one mapping between the compact
task features $h'_{i}$ and their latent representations $z_{i}$.  
This invertibility prevents representational collapse and guarantees that
minimizing the latent-space alignment loss $\mathcal{L}_{\text{align}}$ has a
direct, controllable effect on the alignment of the underlying task features.
The following lemma establishes this relationship for any pair of tasks
$t_1$ and $t_2$.

\begin{lemma}[Bounded Feature Discrepancy]
\label{lemma:bound}
Let $h'_{1}, h'_{2} \in \mathbb{R}^N$ be the compact features of two tasks
$t_1$ and $t_2$ passed into their coupling layers 
$c_{1}$ and $c_{2}$.
Assume their inverse transformations 
$c_{1}^{-1}$ and $c_{2}^{-1}$ 
are $L$-Lipschitz continuous with constant $L$%
~\cite{virmaux2018lipschitz,gouk2021regularisation}.  
Then the L2 distance between the original features is upper-bounded by
\begin{equation}
\| h'_{1} - h'_{2} \|_2 
\;\le\;
L \cdot \sqrt{\mathcal{L}_{\mathrm{align}}}
\;+\; \delta,
\end{equation}
where $\delta$ denotes the maximum structural discrepancy between the two
inverse transformations over the domain of interest.
\end{lemma}

\begin{proof}[Proof]
We rewrite the feature discrepancy as
\[
\| h'_{1} - h'_{2} \|_2
=
\| c_{1}^{-1}(z_{1})
   - 
   c_{2}^{-1}(z_{2}) 
\|_2.
\]
Applying the triangle inequality yields
\begin{equation}
\begin{split}
\| h'_{1} - h'_{2} \|_2 
\;\le\;
\| c_{1}^{-1}(z_{1})
   - 
   c_{1}^{-1}(z_{2}) 
\|_2
\;+\;
\\
\| c_{1}^{-1}(z_{2})
   -
   c_{2}^{-1}(z_{2})
\|_2.
\label{eq:triangle_proof_general}
\end{split}
\end{equation}

The first term is bounded by the Lipschitz continuity of 
$c_{1}^{-1}$:
\[
\| c_{1}^{-1}(z_{1}) 
    - 
    c_{1}^{-1}(z_{2}) 
\|_2
\;\le\;
L \cdot \| z_{1} - z_{2} \|_2
=
L \cdot \sqrt{\mathcal{L}_{\mathrm{align}}}.
\]

The second term depends only on the structural difference between the two 
inverse networks when evaluated at the same latent representation.
Thus it is bounded by a constant $\delta$.

Combining these bounds yields the result.
\end{proof}

This lemma provides a theoretical guarantee: minimizing the alignment loss
$\mathcal{L}_{\mathrm{align}}$ in the latent space directly reduces an upper
bound on the distance between the original task feature distributions.  
Thus, NexusFlow induces provable cross-task feature agreement while preserving
the expressiveness of each task head. 

More specifically, when two tasks are aligned in the latent space via invertible coupling layers, this alignment provably transfers back to the original task-specific feature spaces, up to a controllable error. 


\subsection{Practical Analysis}\label{ssec:analysis}

Building upon our theoretical guarantee, we provide a practical analysis to empirically verify the core mechanism of NexusFlow. We hypothesize that by explicitly aligning the latent feature distributions, our method forges a unified and more effective representation space for knowledge transfer. To validate this, we investigate two key properties of the learned representations: their \textit{alignment} and \textit{intrinsic dimensionality}. We assess alignment both qualitatively through t-SNE visualizations~\cite{maaten2008visualizing,cai2022theoretical} and quantitatively using the Maximum Mean Discrepancy (MMD) metric~\cite{gretton2012kernel,zhang2025gps}. We then analyze their intrinsic dimensionality via Principal Component Analysis (PCA)~\cite{abdi2010principal,del2021effective} to ensure the unified representation is also complex enough to serve both disparate tasks.

\begin{table*}[h]
    \centering	
    \caption{MMD score comparison in two datasets on same training setting (smaller means greater similarity).}
    \vspace{-3mm}
    \resizebox{0.7\textwidth}{!}{
 \begin{tabular}{c|cccccc}
        \toprule
        & Ref & Baseline & MTPSL & JTR & Ours (w/o inv) &Ours \\
        \midrule
        MMD(nuScenes) & 0.23$\pm$0.32 & 2.97$\pm$0.54  & 2.81$\pm$0.36 & 2.77$\pm$0.54 
        & 2.54$\pm$0.29
        & \textbf{1.56$\pm$0.47} \\
        \midrule
        MMD(NYU-V2) & 1.76$\pm$0.25 & 4.48$\pm$0.49 & 3.76$\pm$0.41 & $\backslash$  &
        $\backslash$
        & \textbf{3.02$\pm$0.36} \\        
        \bottomrule
    \end{tabular} 
    }
    \vspace{-0.3em}

    
    \label{tab:SSL-results}
    \vspace{-1mm}
\end{table*}    

\paragraph{Distribution Alignment Analysis.}

We first evaluate the alignment of the feature distributions. For a qualitative understanding, we employ t-SNE~\cite{maaten2008visualizing,cai2022theoretical} to visualize two (or three) sets of high-dimensional features: the latent variables produced by our NexusFlow module. As illustrated in Figure~\ref{fig:tsne_analysis} and Figure~\ref{fig:tsne_analysis_2}, the visualization is revealing. The critical distinction shows the latent variables after the our designed surrogate networks. Here, the features from the baseline methods still exhibit a clear distributional shift. In contrast, the latent variables from our NexusFlow model become substantially intermingled, providing visual evidence that our method creates a more unified manifold for knowledge transfer.

This qualitative observation is substantiated by our quantitative analysis using the Maximum Mean Discrepancy (MMD). MMD is a standard metric for measuring the distance between two distributions, where a lower score signifies greater similarity~\cite{gretton2012kernel,yue2024understanding,chen2024fair}. We compute the MMD between the intermediate features(after NexusFlow module) from mapping-annotated samples and tracking-annotated samples on nuScenes dataset, we also compute the MMD among segmentation-annotated samples, depth-annotated samples and normal-annotated samples on NYU-V2 datasets. As shown in Table~\ref{tab:SSL-results}, on both datasets, NexusFlow module leads to a significant reduction in the MMD score compared to all baselines. This quantitatively validates that NexusFlow effectively reduces the distributional distance between the disparate tasks, creating the statistical foundation for knowledge transfer.

\begin{figure}[h]
\centering
\includegraphics[width=0.85\linewidth]
{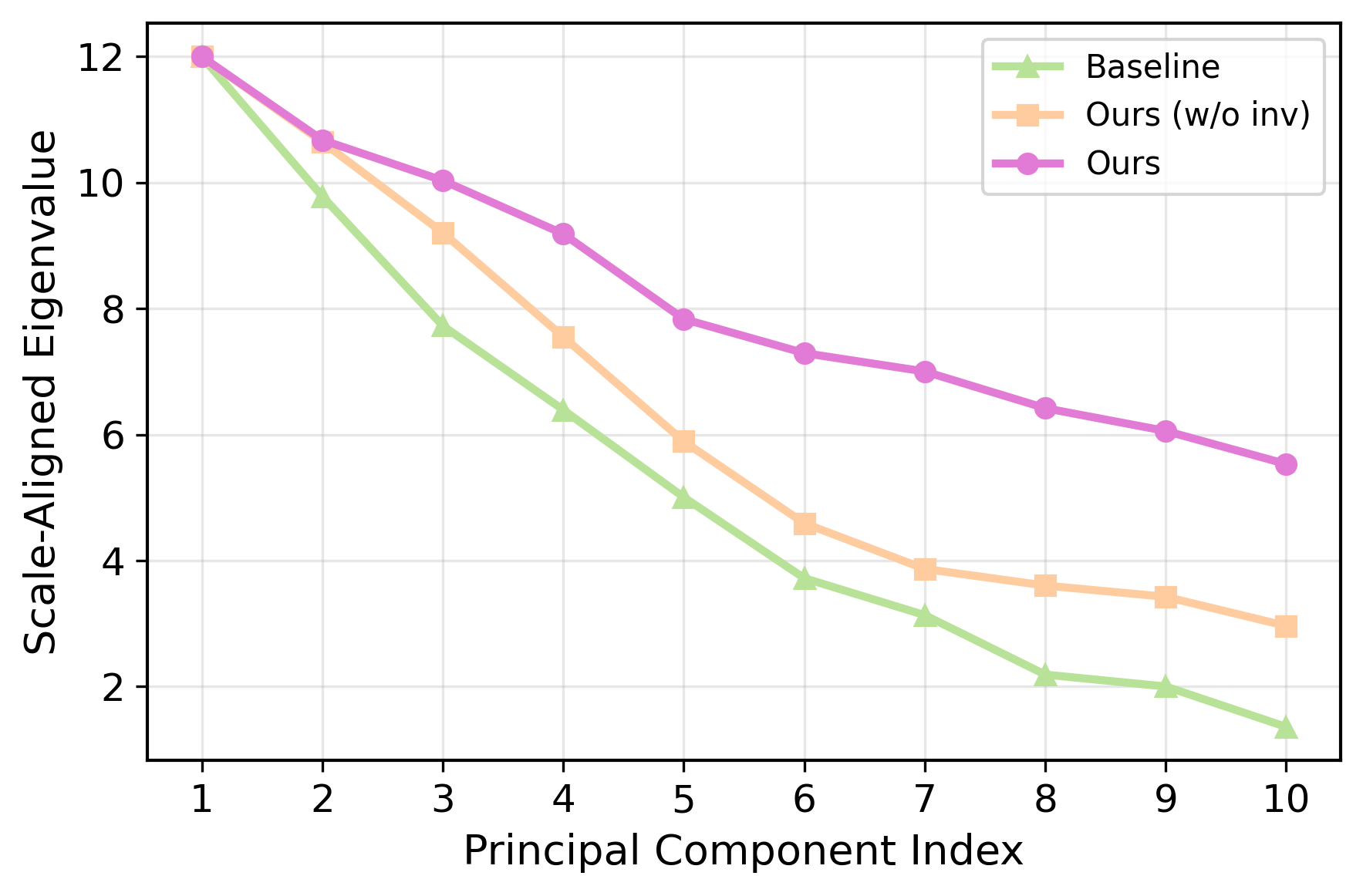}
\vspace{-0.8em}
\caption{\small{Figure of eigenvalue magnitudes decay. Slower decay indicate forging more complex and information-rich feature space.}}
\vspace{-2.0em}
\label{fig:pca}
\end{figure}

\paragraph{Intrinsic Dimensionality Analysis.}
Beyond showing the distributions are closer, we analyze the feature distributions' intrinsic complexity using Principal Component Analysis (PCA). Our analysis is guided by the premise that a representation with a slower decay in its eigenvalue magnitudes holds more informative dimensions, making it more suitable for complex, multi-tasks~\cite{goel2017eigenvalue,ansuini2019intrinsic,kim2023vne,zhang2025gps}. As shown in the Scree plot~\cite{zhu2006automatic} in Figure~\ref{fig:pca}, the features from the baseline and a variant of our model where we ablate the invertible layer (`Ours (w/o inv)`) exhibit rapid eigenvalue decay, suggesting their representations are more compressible and may lose nuanced information. In contrast, the eigenvalue of our full NexusFlow model decays at a slower rate. This provides evidence that our alignment process forges a more complex and information-rich feature space, one that retains the high dimensionality necessary to effectively serve two structurally disparate tasks.

\section{Experiments}
\label{sec:exp}

Our experimental evaluation is structured in two parts, corresponding to the problem settings defined in Section ~\ref{prob_def}.

\subsection{Core Challenge on nuScenes}
This setting addresses our central challenge of domain-partitioned, structurally disparate tasks.

\bfsection{Dataset and PS-MTL Protocol}
We use nuScenes~\cite{caesar2020nuscenes}, a large-scale autonomous driving dataset captured in Boston and Singapore. It provides annotations for our two disparate tasks: multi-object tracking ($\mathcal{T}_{track}$) and map reconstruction ($\mathcal{T}_{map}$). To rigorously evaluate the performance under an explicit domain gap, we design a specific data protocol: ground-truth annotations for the \textbf{\textit{mapping}} task are \textit{only} provided for scenes in \textbf{Boston}, while annotations for the \textbf{\textit{tracking}} task are \textit{only} provided for scenes in \textbf{Singapore}. This challenging, geographically-partitioned supervision protocol is applied consistently across all evaluated methods.

\bfsection{Baselines and Architecture}
To ensure a fair comparison, all methods are built upon the SOTA \textbf{UniAD}~\cite{hu2023planning} architecture as the backbone. We compare against several methods: (1) \textbf{Full-supervision }, trained with full supervision for both tasks from both cities. (2) \textbf{Baseline}, representing the standard UniAD model trained on our domain-partitioned data. (3) Two re-implemented PS-MTL methods, \textbf{MTPSL}~\cite{li2022learning} and \textbf{JTR}~\cite{nishi2024joint}, whose core mechanisms we adapt and integrate into the UniAD framework.

\begin{table*}[t]
    \centering

  \caption{\small{\textbf{Multi-object tracking and Online mapping results.}} Ours achieves competitive performance against Sota methods on both tasks. For online mapping, we report segmentation IoU (\%).}
  \vspace{-3mm}
    
    \scalebox{0.8}{
    \begin{tabular}{l|cccc|cccc}
        \toprule
        \multirow{2}{*}{Method} & \multicolumn{4}{c|}{Multi-object Tracking} & \multicolumn{4}{c}{Online Mapping} \\
        \cmidrule{2-9}
        & \cellcolor{gray!30}AMOTA$\uparrow$ & AMOTP$\downarrow$ & Recall$\uparrow$ & IDS$\downarrow$ & \cellcolor{gray!30}Lanes$\uparrow$ & Drivable$\uparrow$ & Divider$\uparrow$ & Crossing$\uparrow$ \\
        \midrule
        Full-supervision & \cellcolor{gray!30}{0.323} & {1.328} & {0.431} & {696} & \cellcolor{gray!30}31.4 & {71.42} & {34.4} & {21.3} \\
        \midrule
        Baseline~\cite{hu2023planning} & \cellcolor{gray!30}{0.289} & {1.488} & {0.362} & {1025} & \cellcolor{gray!30}27.1 & 62.7 & 22.6 & 14.1 \\
        MTPSL~\cite{li2022learning} & \cellcolor{gray!30}0.255 & 1.504 & 0.321 & 1089 & \cellcolor{gray!30}27.0 & 59.6 & 21.7 & 11.5 \\
        JTR~\cite{nishi2024joint} & \cellcolor{gray!30}0.197 & 1.547 & 0.317 & 774 & \cellcolor{gray!30}25.1 & 57.6 & 21.0 & 12.1 \\
        \textbf{Ours (One stage)} & \cellcolor{gray!30}0.318 & 1.353 & 0.407 & 734 & \cellcolor{gray!30}37.0 & 63.7 & 29.4 & 21.5 \\
        \textbf{Ours (Two stage)} & \cellcolor{gray!30}\textbf{0.329} & \textbf{1.322} & \textbf{0.428} & \textbf{690} & \cellcolor{gray!30}\textbf{37.1} & \textbf{64.5} & \textbf{30.0} & \textbf{22.8} \\
        \bottomrule
    \end{tabular}
    }
    \vspace{-3pt}

    \label{tab:sota-combined}
    \vspace{-10pt}
\end{table*}

\begin{table*}[h]
    \centering
    \caption{\textbf{Ablation study on multi-object tracking and online mapping.} Ours with 6 layers achieves the best performance on both tasks.}    
    \scalebox{0.8}{
    \begin{tabular}{l|cccc|cccc}
        \toprule
        \multirow{2}{*}{Method} & \multicolumn{4}{c|}{Multi-object Tracking} & \multicolumn{4}{c}{Online Mapping (IoU \%)} \\
        \cmidrule{2-9}
        & \cellcolor{gray!30}AMOTA$\uparrow$ & AMOTP$\downarrow$ & Recall$\uparrow$ & IDS$\downarrow$ & \cellcolor{gray!30}Lanes$\uparrow$ & Drivable$\uparrow$ & Divider$\uparrow$ & Crossing$\uparrow$ \\
        \midrule
        Baseline~\cite{hu2023planning} & \cellcolor{gray!30}{0.289} & {1.488} & {0.362} & {1025} & \cellcolor{gray!30}27.1 & 62.7 & 22.6 & 14.1 \\
        Ours (w/o inv) & \cellcolor{gray!30}0.214 & 1.507 & 1.355 & 731 & \cellcolor{gray!30}32.3 & 56.8 & 27.7 & 21.9 \\
        Ours (1 Layer) & \cellcolor{gray!30}0.236 & 1.482 & 0.389 & 982 & \cellcolor{gray!30}32.9 & 57.2 & 28.1 & 20.9 \\
        Ours (2 Layer) & \cellcolor{gray!30}0.258 & 1.475 & 0.396 & 913 & \cellcolor{gray!30}33.5 & 58.2 & 28.3 & 21.0 \\
        Ours (4 Layer) & \cellcolor{gray!30}0.292 & 1.405 & 0.402 & 854 & \cellcolor{gray!30}35.3 & 60.3 & 28.6 & 21.1 \\
        \textbf{Ours (6 Layer)} & \cellcolor{gray!30}\textbf{0.329} & \textbf{1.322} & \textbf{0.428} & \textbf{690} & \cellcolor{gray!30}\textbf{37.1} & \textbf{64.5} & \textbf{30.0} & \textbf{22.8} \\
        Ours (8 Layer) & \cellcolor{gray!30}0.247 & 1.453 & 0.385 & 1054 & \cellcolor{gray!30}33.1 & 57.9 & 28.5 & 20.7 \\
        \bottomrule
    \end{tabular}
    }
    \vspace{-3pt}

    \label{tab:ablation-combined}
    \vspace{-10pt}
\end{table*}

\bfsection{Implementation Details}
For the nuScenes experiments, all methods are trained from scratch following the default setting in \textbf{UniAD}~\cite{hu2023planning} on four NVIDIA A100 80G GPUs. Our proposed \textbf{NexusFlow} module, implemented with $N=4$ invertible coupling layers, is inserted between the shared BEV encoder and the two disparate task decoders to align their latent feature distributions.

\begin{figure}[t]
  \centering
  \includegraphics[width=0.47\textwidth]{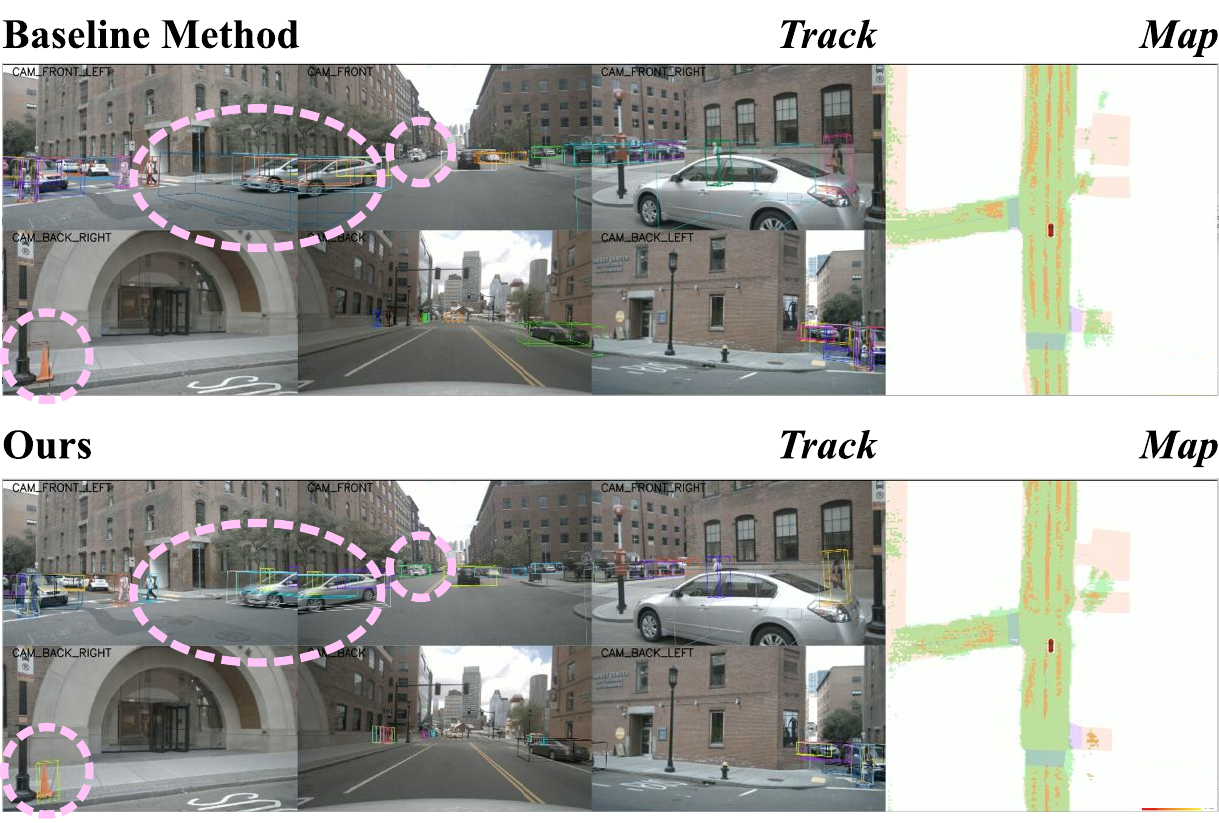}
  \vspace{-5pt}
  \caption{We show visualization results for multi-object tracking and online mapping tasks in surround-view images and BEV from left to right. We highlight the difference with red dashed circle.
  }
  \label{fig:vis}
\vspace{-20pt}
\end{figure}

\bfsection{Quantitative Results}
We present our results in Table~\ref{tab:sota-combined}, the main metric for each task highlighted in gray. Training is strictly conducted under our scenario-based partial supervision protocol, evaluation is performed on the complete validation set, including scenes from both Boston and Singapore. Our proposed NexusFlow significantly outperforms all partially supervised baselines and, remarkably, achieves performance competitive with the Full-supervision.

\begin{itemize}[nosep, leftmargin=*]

\item {\it Multi-Object Tracking:} For the tracking task (left side of the table), NexusFlow sets a new SOTA for this challenging PS-MTL setting. It surpasses MTPSL~\cite{li2022learning} by a large margin of \textbf{+7.4\% AMOTA} and improves upon the standard Baseline~\cite{hu2023planning} by \textbf{+4.0\% AMOTA}. Furthermore, our method achieves the lowest ID Switch score, demonstrating a superior ability to maintain temporal consistency for each tracklet.

\item {\it Online Mapping:} For the mapping task (right side of the table), our method shows substantial gains in segmenting crucial road elements. Notably, it outperforms the Baseline \cite{hu2023planning} and MTPSL \cite{li2022learning} on lanes by over \textbf{+10\% IoU}, a critical component for safe downstream motion planning. This result confirms that effective knowledge was transferred from the tracking-annotated domain (Singapore) to enhance the mapping performance.


\end{itemize}

Meanwhile, we observe that prior state-of-the-art PS-MTL methods, \textbf{MTPSL}~\cite{li2022learning} and \textbf{JTR}~\cite{nishi2024joint}, designed for homogeneous dense prediction tasks can even degrade baseline performance in our setting. This highlights that the more realistic scenario of structurally disparate, domain-partitioned partial labels can break methodologies tailored only for small-scale, homogeneous multi-task learning.

\bfsection{Qualitative Results} These results are corroborated by our qualitative analysis in Figure~\ref{fig:vis}. We observe that our model is able to detect objects more accurately. A full video comparison is provided in the supplementary material.

\bfsection{Ablation Study}
We conduct an ablation study to validate our design choices, with results presented in Table~\ref{tab:ablation-combined}. We compare our full model against two variants: the \textbf{Baseline} (without the NexusFlow module) and \textbf{Ours (w/o inv)} (without the invertible coupling layers). The results confirm that our full model significantly outperforms both variants. While \textbf{Ours (w/o inv)} surpasses the \textbf{Baseline}, the performance gap to the full model underscores the critical role of the invertible transformation. Furthermore, we find that a coupling layer depth of 6 achieves the optimal performance, which we adopt in our final design.

\subsection{N-Task Challenge on NYU-V2}
\bfsection{Dataset and PS-MTL Protocol}
Although our main experiments focus on the two-task setting, we further evaluate NexusFlow under a more general $N$-task MTL framework.  
Due to constraints on data availability, we consider a simpler homogeneous dense-prediction setup and adopt the \textbf{NYU-V2} dataset~\cite{silberman2012indoor}, which provides three standard tasks: 13-class semantic segmentation, depth estimation, and surface normal prediction.  
Following common practice, all images are resized to $288 \times 384$~\cite{liu2019end}.
To simulate the PS-MTL setting, instead of randomly selecting a single task label per sample and masking the losses of the others~\cite{ruder2019latent,li2022learning}, we partition the entire dataset into three roughly equal subsets using clustering, and assign exactly one task’s ground-truth label to each subset. This enforces a realistic partial supervision regime in which only one task annotation is available per training example.





\noindent
\bfsection{Baselines and Architecture}
We adopt a standard MTL architecture, which is common for homogeneous dense-prediction tasks~\cite{li2022learning}. This consists of a  Multi-Task Attention Network~\cite{liu2019end}, including a feature encoder followed by three lightweight, task-specific decoders. We compare \textbf{NexusFlow} against several key baselines: (1) \textbf{Full-supervision }, trained with full supervision for all three tasks. (2) \textbf{Baseline}, representing the standard shared-encoder model trained \textit{only} on our PS-MTL Protocol. (3) \textbf{MTPSL}~\cite{li2022learning} with our PS-MTL protocol.

\begin{table}[t]
	\centering
		\caption{\small{Semantic segmentation, depth estimation, and surface normals estimation results on NYU-v2. Ours achieves the best performance on all tasks in the partially supervised MTL setting.}}
    
    \resizebox{0.47\textwidth}{!}
    {
		\begin{tabular}{lccccccccccc}

		    \toprule
		    Method & Seg. (IoU) $\uparrow$ & Depth (aErr) $\downarrow$ & Norm. (mErr) $\downarrow$ \\
		    \midrule
		    Full-supervision   & 35.84 & 0.5694 & 30.42 \\
		    \midrule
		    Baseline & 26.75 & 0.6511 & 35.17 \\
            MTPSL~\cite{li2022learning}  & 29.65 & 0.6286 & 33.28 \\
            Ours ($\mathcal{L}_{\text{alignpair}}$) & {\bf 31.35} & {\bf 0.6082} & {\bf 31.74} \\
            Ours ($\mathcal{L}_{\text{aligncenter}}$) & {\bf 31.70} & {\bf 0.6055} & {\bf 31.88} \\
			\bottomrule
		\end{tabular}%
			}
		\vspace{-0.1cm}

		\label{tab:nyuv2full}
\vspace{-15pt}
\end{table}%

\bfsection{Implementation Details}
For the NYU-V2 experiments, all methods are trained from scratch following the default setting in MTPSL~\cite{li2022learning} on one NVIDIA A100 80G GPU. Our proposed \textbf{NexusFlow} module, implemented with $N=3$ invertible coupling layers, is inserted between the feature encoder and the three task decoders to align their latent feature distributions.

\bfsection{Quantitative Results}
We present our results in Table~\ref{tab:nyuv2full}. Our proposed NexusFlow significantly outperforms all partially supervised baselines. Although still left behind with the Full-supervision, it has a strong advantage in computational cost, especially when compared with MTPSL~\cite{li2022learning} in terms of training time and GPU memory (shown in Table~\ref{tab:performance}).

\bfsection{Qualitative Result} We show qualitative comparisons in Figure~\ref{fig:vis2} for all three tasks. We observe that our framework produces more accurate predictions in every task, especially in the semantic segmentation and depth estimation tasks.

\begin{table}[t!]
	\centering

	\caption{\textbf{Efficiency Comparison.} Our method matches the fully supervised MTL and partially supervised MTL baseline in both training time and GPU memory usage, while achieving nearly a 50\% reduction in memory cost compared to MTPSL.}
    
    \resizebox{0.42\textwidth}{!}
    {
		\begin{tabular}{lcc}
			
		    \toprule
		    Method & Training Time $\downarrow$ & GPU Memory $\downarrow$\\
		    \midrule
		    Full-supervision   & 8$\sim$hours & 3906 MiB \\
              \midrule
		    Baseline & 8$\sim$hours & 3906 MiB \\
            MTPSL~\cite{li2022learning} & 10$\sim$hours & 8900 MiB \\
		    Ours & {\bf 8$\sim$hours} & {\bf 4554 MiB} \\
			\bottomrule
		\end{tabular}%
	}
	\vspace{-0.2cm}

	\label{tab:performance}
\end{table}%

\begin{figure}[t]
  \centering
  \includegraphics[width=0.47\textwidth]{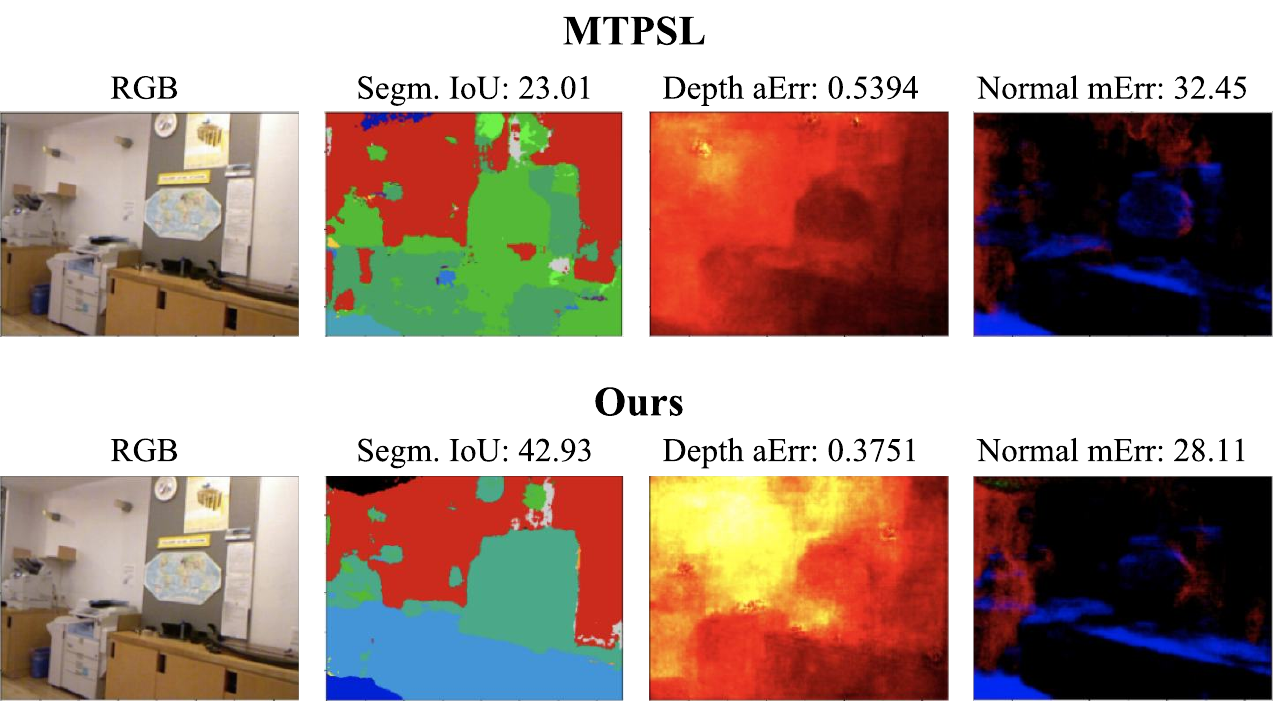}
  \vspace{-5pt}
  \caption{\textbf{Qualitative results on NYU-v2.} The fist column shows the RGB image, the second column plots the predictions with the IoU ($\uparrow$) score of all methods for semantic segmentation, the third column presents the predictions of depth estimation with the absolute error ($\downarrow$), and we show the prediction of surface normal with mean error ($\downarrow$) in the last column.}
  \label{fig:vis2}
\vspace{-20pt}
\end{figure}

\section{Conclusion}

\vspace{-2mm}
In this paper, we formally addressed the previously underexplored setting of Partially Supervised Multi-Task Learning (PS-MTL) involving structurally disparate tasks with domain-partitioned supervision. To tackle this challenge, we introduced \textbf{NexusFlow}, a lightweight and plug-and-play module that aligns task-specific feature distributions using independent invertible coupling layers, enabling effective knowledge transfer across heterogeneous tasks. Our theoretical and empirical analyses show that these invertible transformations offer stable and reliable cross-task alignment while avoiding dimensionality collapse. Extensive experiments on both autonomous driving and indoor dense prediction benchmarks further demonstrate the strong effectiveness and broad generalization capability of NexusFlow.

\newpage
\clearpage
{
    \small
    \bibliographystyle{ieeenat_fullname}
    \bibliography{main}
}

\clearpage
\setcounter{page}{1}
\onecolumn 
\begin{center}
    {\LARGE \textbf{Supplementary Materials}}
\end{center}

\appendix
\counterwithin{figure}{section} 
\renewcommand{\thefigure}{\thesection.\arabic{figure}}

\counterwithin{equation}{section}
\counterwithin{table}{section}

\renewcommand{\theequation}{\thesection.\arabic{equation}}
\renewcommand{\thetable}{\thesection.\arabic{table}}

\section{Supplementary experiments on nuScenes}

We conduct comprehensive supplementary experiments to further validate our design choices. First, we investigate the impact of different feature aggregation architectures: comparing MLPs with deformable attention layers (defo.), with results summarized in Table~\ref{tab:supplementary-nuScenes1}. We observe that as long as the key coupling layer is present, performance remains stable, and architectures equipped with deformable attention achieve even stronger results. In contrast, removing the coupling layer leads to a substantial performance drop. These findings further confirm that our invertible transformations provide reliable alignment of feature distributions across heterogeneous tasks while effectively mitigating dimensionality collapse.

We also evaluate two geographically partitioned supervision settings. The Standard Geo-Split (S-Geo) setup follows our main protocol, where mapping supervision is available only in \textbf{Boston} and tracking supervision only in \textbf{Singapore}. We additionally introduce a Reversed Geo-Split (R-Geo) configuration that swaps the supervision visibility across regions. This complementary variant allows us to assess the consistency of our method under geographically perturbed supervision, with results reported in Table~\ref{tab:supplementary-nuScenes2}.
\vspace{5pt}
\begin{table}[!h] 
    \centering
    \vspace{10pt}
    \caption{\textbf{Supplementary experiments on nuScenes (Architecture).} We evaluate different settings on nuScenes dataset with different feature aggregation architectures.}    
    \scalebox{0.75}{ 
    \begin{tabular}{l|cccc|cccc}
        \toprule
        \multirow{2}{*}{Method} & \multicolumn{4}{c|}{Multi-object Tracking} & \multicolumn{4}{c}{Online Mapping (IoU \%)} \\
        \cmidrule{2-9}
        & \cellcolor{gray!30}AMOTA$\uparrow$ & AMOTP$\downarrow$ & Recall$\uparrow$ & IDS$\downarrow$ & \cellcolor{gray!30}Lanes$\uparrow$ & Drivable$\uparrow$ & Divider$\uparrow$ & Crossing$\uparrow$ \\
        \midrule
        Baseline & \cellcolor{gray!30}{0.289} & {1.488} & {0.362} & {1025} & \cellcolor{gray!30}27.1 & 62.7 & 22.6 & 14.1 \\
        Ours (MLP w/o Inv) & \cellcolor{gray!30}0.214 & 1.507 & 1.355 & 731 & \cellcolor{gray!30}32.3 & 56.8 & 27.7 & 21.9 \\
        Ours (MLP + Inv) & \cellcolor{gray!30}0.316& 1.250 & 0.397 & 879 & \cellcolor{gray!30}36.2 & 62.31 & 28.80 & 20.75 \\
        \textbf{Ours (defo. + Inv)} & \cellcolor{gray!30}\textbf{0.329} & \textbf{1.322} & \textbf{0.428} & \textbf{690} & \cellcolor{gray!30}\textbf{37.1} & \textbf{64.5} & \textbf{30.0} & \textbf{22.8} \\
        \bottomrule
    \end{tabular}
    }
    \label{tab:supplementary-nuScenes1}


    \caption{\textbf{Supplementary experiments on nuScenes (Geo-Split).} We evaluate two geographically partitioned supervision settings on nuScenes dataset.}    
    \scalebox{0.75}{
    \begin{tabular}{l|cccc|cccc}
        \toprule
        \multirow{2}{*}{Method} & \multicolumn{4}{c|}{Multi-object Tracking} & \multicolumn{4}{c}{Online Mapping (IoU \%)} \\
        \cmidrule{2-9}
        & \cellcolor{gray!30}AMOTA$\uparrow$ & AMOTP$\downarrow$ & Recall$\uparrow$ & IDS$\downarrow$ & \cellcolor{gray!30}Lanes$\uparrow$ & Drivable$\uparrow$ & Divider$\uparrow$ & Crossing$\uparrow$ \\
        \midrule
        Baseline & \cellcolor{gray!30}{0.289} & {1.488} & {0.362} & {1025} & \cellcolor{gray!30}27.1 & 62.7 & 22.6 & 14.1 \\
        Ours (R-Geo) & \cellcolor{gray!30} \textbf{0.332}& \textbf{1.304} & \textbf{0.435} & \textbf{653} & \cellcolor{gray!30}36.6 & 63.8 & 29.7 & 22.5 \\
        \textbf{Ours (S-Geo)} & \cellcolor{gray!30}{0.329} & {1.322} & {0.428} & {690} & \cellcolor{gray!30}\textbf{37.1} & \textbf{64.5} & \textbf{30.0} & \textbf{22.8} \\
        \bottomrule
    \end{tabular}
    }
    \label{tab:supplementary-nuScenes2}
    \vspace{-5pt} 
\end{table}
\vspace{5pt}
\section{Supplementary experiments on NYU-V2}




We conduct additional supplementary experiments to validate our design choices on dense prediction tasks. We first examine the impact of different feature aggregation architectures by comparing MLP-based aggregators under the same invertible-layer configuration (3 layers), as shown in Table~\ref{tab:supplementary_nyuv2_left}. We then evaluate the complementary setting in which the feature aggregator is fixed (MLP, 3 layers) while varying the number of invertible layers, with results reported in Table~\ref{tab:supplementary_nyuv2_right}.


\vspace{5pt}
\begin{table}[!h]
    \centering
    
    \begin{minipage}[t]{0.49\textwidth}
        \centering
        \caption{\small{\textbf{NYU-V2 Results Part 1.} We evaluate settings on NYU-V2 dataset with different MLP layers in feature aggregator.}}
        \label{tab:supplementary_nyuv2_left}
        
        \resizebox{\linewidth}{!}{
            \begin{tabular}{lccccccccccc}
                \toprule
                Method & Seg. (IoU) $\uparrow$ & Depth (aErr) $\downarrow$ & Norm. (mErr) $\downarrow$ \\
                \midrule
                Baseline & 26.75 & 0.6511 & 35.17 \\
                1 MLP & 28.57 & 0.6841 & 34.54 \\
                2 MLP & 29.46 &  0.6311 & 33.42 \\
                \textbf{3 MLP} & {\bf 31.70} & {\bf 0.6055} & {\bf 31.88} \\
                4 MLP &  29.50 & 0.6404  & 33.39 \\
                5 MLP & 29.78 & 0.6268  & 33.12 \\            
                6 MLP & 29.85  & 0.6258  & 33.05 \\
                \bottomrule
            \end{tabular}
        }
    \end{minipage}
    \hfill 
    \begin{minipage}[t]{0.49\textwidth}
        \centering
        \caption{\small{\textbf{NYU-V2 Results Part 2.} We evaluate settings on NYU-V2 dataset with different Invertible layers.}}
        \label{tab:supplementary_nyuv2_right}
        
        \resizebox{\linewidth}{!}{
            \begin{tabular}{lccccccccccc}
                \toprule
                Method & Seg. (IoU) $\uparrow$ & Depth (aErr) $\downarrow$ & Norm. (mErr) $\downarrow$ \\
                \midrule
                Baseline & 26.75 & 0.6511 & 35.17 \\
                1 Inv & 28.20 & 0.6300 & 33.50 \\
                2 Inv & 29.50 & 0.6200 & 33.10 \\
                3 Inv & 29.90 & 0.6250 & 33.00 \\                
                \textbf{4 Inv} & \textbf{31.70} & \textbf{0.6055} & \textbf{31.88} \\

                5 Inv & 29.80 & 0.6280 & 33.20 \\
                6 Inv & 27.10 & 0.6400 & 34.00 \\
                \bottomrule
            \end{tabular}
        }
    \end{minipage}

    \vspace{-5pt} 
\end{table}



\end{document}